\documentclass[10pt,journal]{IEEEtran}

\usepackage[pdftex]{graphicx}
\DeclareGraphicsExtensions{.pdf,.jpeg,.png}

\usepackage{multirow,setspace,verbatim,amsfonts,graphicx,amsmath,amsthm,amsbsy,amssymb,epsfig,url}
\usepackage{amsthm}
\usepackage{gensymb}
\usepackage{graphicx}

\usepackage{amsfonts}
\usepackage{amsmath}
\usepackage{amssymb}
\usepackage{amsthm}
\usepackage{tikz,graphicx}
\usepackage{mathrsfs}

\usepackage{algorithmicx}
\usepackage{algorithm,algpseudocode}

\usepackage{array}
\usepackage{booktabs}

\newcommand{\fix}{\mathrm{Fix}}

\newcommand{\Hc}{{\mathcal H}}

\newcommand{\Qc}{{\mathcal Q}}

\newcommand{\Wc}{{\mathcal W}}

\newcommand{\Rd}{{\mathbb R}}

\newcommand{\hank}{\mathbb{H}}

\newcommand{\rank}{\textsc{rank}}

\newtheorem{theorem}{Theorem}

\begin{document}

\title{ Deep Convolutional Framelet Denosing for Low-Dose CT via Wavelet Residual Network}

\author{Eunhee~Kang,~
        Won Chang,
        Jaejun~Yoo,~
        and~Jong~Chul~Ye$^{*}$,~\IEEEmembership{Senior Member,~IEEE}
\thanks{EK, JY and JCY are with the Department of Bio and Brain Engineering, Korea Advanced Institute of Science and Technology (KAIST), Daejeon 34141, Republic of Korea (e-mail: \{eunheekang,jaejun2004,jong.ye\}@kaist.ac.kr).
WC is with Department of Radiology, Seoul National University Bundang Hospital,  Gyeonggi-do, 13620, Republic of Korea (email: sword1981@hanmail.net).}
\thanks{Part of this work was presented in 2017 International Conference on Fully Three-Dimensional Image Reconstruction in Radiology and Nuclear Medicine. The trained network and test data are available in  https://github.com/eunh/low\_dose\_CT.}
}		


\maketitle

\begin{abstract}

Model based iterative reconstruction (MBIR) algorithms for low-dose X-ray CT are computationally expensive.
To address this problem, we recently  proposed a deep convolutional neural network (CNN) for low-dose X-ray CT 
and won the second place in 2016 AAPM Low-Dose CT Grand Challenge.
However, some of the texture were not fully recovered. 
To address this problem, here we propose a novel  framelet-based denoising algorithm using wavelet residual network 
which synergistically combines the expressive power of deep learning and the performance guarantee from the framelet-based denoising algorithms.  
The new algorithms were inspired by the recent interpretation of the deep convolutional neural network (CNN) as a cascaded convolution framelet signal representation. 
Extensive experimental results confirm that the proposed networks have significantly improved  performance and preserves the detail texture of the original images.

\end{abstract}


\IEEEpeerreviewmaketitle

\section{Introduction}
\label{sec:introduction}

\IEEEPARstart{X}{-ray} computed tomography (CT) is one of the most valuable imaging techniques in clinics.
It is used in various ways, including whole-body  diagnostic CT, C-arm CT for interventional imaging,  dental CT, etc.
However, X-ray CT causes potential cancer risks due to radiation exposure.
To ensure patient safety,   X-ray dose reduction techniques have been extensively studied,
and  the reduction in the number of X-ray photons using  tube current modulation is considered one of the solutions. 
A drawback of this approach is, however, the low signal-to-noise ratio (SNR) of projections, which induces noise in the reconstructed image.
Various model based iterative reconstruction (MBIR) methods \cite{beister2012iterative,ramani2012splitting,sidky2008image} have been investigated to obtain a clear reconstructed image.
However, these approaches are usually computationally expensive due to  the iterative applications of forward and backward projections.

Recently, deep learning approaches have been actively explored for various computer vision applications through the use of  extensive data and powerful graphical processing units (GPUs).
Deep networks have achieved  great successes in  computer vision applications such as classification \cite{krizhevsky2012imagenet}, denoising \cite{burger2012image,mao2016image,zhang2016beyond}, segmentation \cite{ronneberger2015u}, and super-resolution \cite{kim2016accurate}, etc.

In MR image reconstruction,  Wang et al~\cite{wang2016accelerating}  was the first to apply deep learning to compressed sensing MRI (CS-MRI). 
Deep network architecture using unfolded iterative compressed sensing (CS) algorithm was also proposed 
\cite{hammernik2016learning,sun2016deep}. 
In CT restoration problems, our group  introduced the deep learning approach for low-dose X-ray CT \cite{kang2017deep},
whose performance has been rigorously confirmed by winning the second place award in 2016 AAPM Low-Dose CT Grand Challenge.
Since then, several pioneering deep learning approaches for low-dose CT have been proposed by many researchers \cite{chen2017low,chen2017lowBOE,adler2017learned,chen2017learned,wurfl2016deep,yang2017ct,wang2016perspective,yang2017low,wolterink2017generative,wu2017iterative}.
Some algorithms uses generative adversarial network (GAN) loss \cite{yang2017low, wolterink2017generative}.
 Recent proposal is to incorporate deep neural network within iterative steps \cite{adler2017learned,gupta2017cnn}.
However, existing algorithms consider a deep network as a black-box,
so it is difficult to understand the role of deep networks within iterative steps.

Therefore, one of the main contributions of this paper is to show that  a feed-forward deep learning-based denoising  is indeed the first iteration of 
a special instance of frame-based denoising algorithm using {deep convolutional framelets} \cite{ye2017deep}.
Frame-based denoising approaches using wavelet frames have been an extensive research topics in applied
mathematics community due to its proven convergence \cite{li2014wavelet,dong2017image}.
On the other hand, the theory of deep convolutional framelet  \cite{ye2017deep}
was recently proposed to  explain the mathematical origin of deep neural network 
as a multi-layer realization of   the convolution framelets \cite{yin2017tale}.
Accordingly, the main goal of this paper is to synergistically combine the expressive power of deep neural network
and the performance guarantee from the framelet-based denoising algorithms.  
In particular, we show that the performance of the
deep learning-based denoising algorithm can be improved with iterative steps similar to the classical framelet-based 
denoising approaches \cite{cai2008framelet,cai2009convergence}.
Furthermore, we can provide the theoretical guarantee of the algorithm to converge.

Compared to  the recent proposals of learning-based optimization approaches \cite{adler2017learned,gupta2017cnn},  one of the important advantages of our work
is that our deep network is no more a black box but can be optimized for specific  restoration tasks by choosing
optimal framelet representation.
Thus, we can employ an improved wavelet residual network (WavResNet) structure \cite{kang2017wavelet} in our deep convolutional framelet denoising thanks to its effectiveness in 
recovering the directional components.
We confirm our theoretical reasoning using extensive numerical experiments.

\section{Theory}
\label{sec:theory}

%
For simplicity, we derive our theory for 1-D signals, but the extension to 2-D image is straightforward.

\subsection{Frame-based Denoising}

Consider an  {analysis operator} $\Wc$ given by
$\Wc^\top = \begin{bmatrix} w_1 &\cdots & w_{m} \end{bmatrix},$
where the superscript $^{\top}$ denotes the Hermitian transpose and
 $\{w_k\}_{k=1}^m$  is a family of function in a Hilbert space $H$.
Then,  $\{w_k\}_{k=1}^m$    is called a { frame} if it satisfies the following inequality \cite{duffin1952class}:
\begin{eqnarray}\label{eq:framebound}
\alpha \|f\|^2 \leq  \|\Wc f\|^2  \leq \beta\|f\|^2,\quad \forall f \in H ,
\end{eqnarray}
where $\alpha,\beta>0$ are called the frame bounds. 
Then the recovery of the original signal can be done 
from the frame coefficient $c=\Wc f$ using the 
{dual frame} $\tilde \Wc$ satisfying the frame condition:
$\tilde \Wc^\top \Wc = I,$
since
$ f = \tilde \Wc^\top c = \tilde \Wc^\top \Wc f = f.$
This condition is often called {the perfect reconstruction (PR) condition}.
We often call  $\tilde \Wc^\top$  as the {synthesis operator}.
The frame is said to be {tight},  if $\alpha=\beta$ in \eqref{eq:framebound}.    This is equivalent to
 $\tilde \Wc =\Wc$ or
$\Wc^\top\Wc = I$.

Suppose that noisy measurement $g \in \Rd^n$ is given by
$$g=f^* + e$$
where $f^*\in \Rd^n$ is a unknown ground-truth image  and  $e\in \Rd^n$ denotes the noise. 
Then, the classical tight frame-based denosing approaches
 \cite{li2014wavelet,dong2017image} 
 solve the following alternating minimization problem:
\begin{eqnarray}\label{eq:cost}
\min_{f,\alpha} \frac{\mu}{2} \|g - f\|^2+   \frac{1-\mu}{2}\left\{ \|\Wc f - \alpha \|^2 + \lambda \|\alpha \|_1\right\}
\end{eqnarray}
where $\lambda,\mu>0$ denote the regularization parameters.
The corresponding proximal  update equation   is then given by  \cite{li2014wavelet,dong2017image}:
\begin{eqnarray}\label{eq:update}
f_{n+1} &=& \mu  g +  (1-\mu) \Wc^\top T_\lambda\left(\Wc f_n\right)  \ , 
\end{eqnarray}
where  $T_\lambda(\cdot)$ denotes the soft-thresholding operator with the threshold value of $\lambda$, and $f_n$ refers to the $n$-th update.
Thus, the frame-based denoising algorithm in \eqref{eq:update} is designed to remove  insignificant parts of frame coefficients through shrinkage operation, 
by assuming that most of the meaningful signal has large frame coefficients and  noises are distributed
across all frame coefficients.

    One of the most important advantages of this framelet-based denoising
   is its proven convergence  \cite{li2014wavelet,dong2017image}.
    Our goal is thus to exploit the proven convergence  of these approaches for our CNN based low-dose CT denoising.
    Toward this goal, in the next
    section we  show that CNN is closely related to the frame bases.

\subsection{Deep Convolutional Framelets}

Here, the theory of the deep convolutional framelet \cite{ye2017deep} is briefly reviewed to make this work self-contained.
To avoid special treatment of boundary condition, our theory is mainly derived using circular convolution.
Specifically, let $f=[f[1],\cdots, f[n]]^T\in \Rd^n$ be an input signal
and $\overline\psi=[\psi[d],\cdots, \psi[1]]^T\in\Rd^d$ denotes filter represented as the flipped version of vector $\psi$.
Then, the convolution operation  in  CNN 
can be represented using Hankel matrix operation \cite{ye2017deep}.
Specifically, a single-input single-output (SISO) convolution with the filter $\overline\psi$   is given by  a matrix vector multiplication:
\begin{eqnarray}\label{eq:SISO}
y = f\circledast \overline\psi &=&\hank_d(f) \psi \ ,
\end{eqnarray}
where 
 $\hank_d(f)$ is a wrap-around Hankel matrix
 \begin{eqnarray*} 
\hank_d(f) =\left[
        \begin{array}{cccc}
        f[1]  &   f[2] & \cdots   &   f[d]   \\
       f[2]  &   f[3] & \cdots &     f[d+1] \\
           \vdots    & \vdots     &  \ddots    & \vdots    \\
              f[n]  &   f[1] & \cdots &   f[d-1] \\
        \end{array}
    \right] \ . 
    \end{eqnarray*}
Similarly, multi-input multi-output (MIMO) convolution with the matrix input $F:= [f_1\cdots f_p]$ and the multi-channel filter
matrix $\overline\Psi$
\begin{eqnarray}\label{eq:MIMOPsi}
\overline\Psi &:=&\begin{bmatrix}   \overline\psi_1^1 & \cdots &   \overline\psi_q^1  \\ \vdots & \ddots & \vdots \\
\overline\psi_1^p & \cdots &  \overline\psi_q^p 
\end{bmatrix}  \in \Rd^{dp \times q} 
\end{eqnarray}
can be represented as
\begin{eqnarray}\label{eq:multifilter}
Y &=& F \circledast \overline\Psi = \hank_{d|p}\left(F\right) \Psi 
\end{eqnarray}
where $\hank_{d|p}\left(F\right)$ is  an {\em extended Hankel matrix}  by stacking  $p$ Hankel matrices side by side: 
\begin{eqnarray}\label{eq:ehank}
\hank_{d|p}\left(F\right)  := \begin{bmatrix} \hank_d(f_1) & \hank_d(f_2) & \cdots & \hank_d(f_p) \end{bmatrix} 
\end{eqnarray}
In \eqref{eq:MIMOPsi}, 
$\overline\psi_i^j \in \Rd^d, i=1,\cdots, q;j=1,\cdots,p$ refer to the  $j$-th input channel  filters to generate the $i$-th output channel.
%
Note that the convolutional representation using  an extended Hankel matrix  in \eqref{eq:multifilter} is equivalent to the multi-channel
filtering operations commonly used in CNN \cite{ye2017deep}.

%

%
 Let $\Phi=[\phi_1,\cdots,\phi_n]$ and $\tilde \Phi =[\tilde \phi_1,\cdots, \tilde\phi_n] \in \Rd^{n\times n}$  (resp. $ \Psi=[\psi_1,\cdots, \psi_q]$ and $\tilde \Psi =[\tilde \psi_1,\cdots,\tilde\psi_q] \in \Rd^{d\times q}$) are { frames and its duals} satisfying the frame condition:
\begin{eqnarray}
 \tilde \Phi \Phi^\top = I_{n\times n} &,&  \Psi \tilde \Psi^{\top} = I_{d\times d}. \label{eq:id}
 \end{eqnarray}
Accordingly, we can obtain the following matrix identity:
 \begin{eqnarray}
 \hank_d(f) =  \tilde \Phi \Phi^\top \hank_d(f) \Psi \tilde \Psi^{\top} = \tilde \Phi C \tilde \Psi^{\top}
 \end{eqnarray}
 where  $C:=\Phi^\top \hank_d(f) \Psi $ denotes the framelet coefficient.
This results in the following
 encoder-decoder layer structure \cite{ye2017deep}:
 \begin{eqnarray}
 f &=& \left(\tilde\Phi C\right) \circledast \nu(\tilde \Psi), \label{eq:dec}  \\
 C &:=& \Phi^\top \left( f \circledast \overline \Psi\right) \label{eq:enc} 
 \end{eqnarray}
 where $\overline\Psi$ is from  \eqref{eq:MIMOPsi} by setting $q=1$, 
and
\begin{eqnarray}
\nu(\tilde\Psi) := \frac{1}{d} \begin{bmatrix} \tilde \psi_1 \\ \vdots \\
\tilde \psi_r
\end{bmatrix} \ . 
\end{eqnarray}
Similarly, 
for a given  matrix input $Z \in \Rd^{n\times p}$, 
we can also derive the paired encoder-decoder structure \cite{ye2017deep}:
\begin{eqnarray}
C &=&  \Phi^\top \left( Z \circledast  \overline\Psi\right) \label{eq:encZ} \\
Z 
&=& \left(\Phi C\right) \circledast \nu(\tilde \Psi) \label{eq:decZ}
\end{eqnarray}
where  the  encoder filter is given by \eqref{eq:MIMOPsi}  and the decoder  filters is defined by
\begin{eqnarray}\label{eq:tauZ}
\nu(\tilde\Psi) &:=&  \frac{1}{d} \begin{bmatrix}  \tilde \psi_1^1 & \cdots &  \tilde \psi_1^p  \\ \vdots & \ddots & \vdots \\
\tilde \psi_q^1 & \cdots &  \tilde \psi_q^p 
\end{bmatrix}  \in \Rd^{dq \times p}
\end{eqnarray}
such that they satisfy the frame condition
\begin{eqnarray}\label{eq:id2}
\Psi \tilde\Psi^\top = I_{dp\times dp} \quad .
\end{eqnarray}
The simple convolutional framelet expansion using  \eqref{eq:enc}, \eqref{eq:dec},  \eqref{eq:encZ} and \eqref{eq:decZ} is so powerful
that the deep CNN architecture emerges from them.
{\color{black}
Specifically, by inserting the pair
 \eqref{eq:encZ} and \eqref{eq:decZ} between the pair \eqref{eq:enc} and \eqref{eq:dec}, 
 we can derive a deep network structure. 
For more detail, see  \cite{ye2017deep}.
}

\subsection{Deep Convolutional Framelet Denoising}

Now,  note that the computation of our deep convolutional framelet coefficients can be represented by 
analysis operator:
$$\Wc f: = C = \Phi^\top (f\circledast \overline \Psi)$$
whereas the synthesis operator is given by the decoder part of convolution:
$$\tilde \Wc^\top C := (\Phi C) \circledast \nu (\tilde \Psi). $$
If  the frame conditions \eqref{eq:id} or \eqref{eq:id2} are met at each layer,
we can therefore use the classical update algorithm in \eqref{eq:update} for denosing. 
Then, what is the shrinkage operator that corresponds to $T_\lambda(\cdot)$ in \eqref{eq:update}?  
One of the unique aspects of deep convolutional framelets is that
by changing   the number of filter channels, we can 
achieve the shrinkage behaviour  \cite{ye2017deep}.
More specifically,  
low-rank shrinkage behaviour emerges  when the number of output filter channels are not sufficient.
Therefore, the explicit application of the shrinkage operator is no more necessary.

 To understand this claim, consider the following regression problem under low-rank Hankel structured matrix constraint:
\begin{eqnarray}
\min_{f\in \Rd^{n}}  & \|f^* -f\|^2 \notag\\
\mbox{subject to }  &\quad \rank \hank_d(f) \leq r < d .  \label{eq:fcost}
\end{eqnarray}
where $f^*\in \Rd^n$ denotes the ground-truth signal, $r$ is the upper
bound of the rank, and  $\hank_d(f)\in \Rd^{n\times d}$. 
The low-rank  Hankel structured matrix constraint in \eqref{eq:fcost}  is known  for its excellent performance   in image denoising \cite{jin2015sparse+}, artifact removal \cite{jin2016mri} 
and deconvolution \cite{min2015fast}.

A classical approach to address \eqref{eq:fcost} is using the explicit singular value shrinkage operation to 
impose the low-rankness \cite{candes2009exact,cai2010singular}.
However, using deep convolutional framelets, we do not need such explicit shrinkage operation.
More specifically,
let $V\in \Rd^{d\times r}$ denote the basis for $R\left((\hank_d(f))^\top\right)$
where $R(\cdot)$ denote the range space. 
Then,  there always exist two matrices pairs $\Phi, \tilde \Phi \in \Rd^{n\times n}$ and $\Psi$, $\tilde \Psi\in \Rd^{d\times r}$ satisfying the conditions
\begin{eqnarray}
\tilde \Phi \Phi^\top = I_{n\times n}, \qquad \Psi \tilde \Psi^{\top} = P_{R(V)}  \label{eq:noframe}
\end{eqnarray}
where $R(V)$ denote the range space of $V$ and $P_{R(V)}$ represents a projection onto $R(V)$.
Note that the bases matrix $\tilde \Psi\in \Rd^{d\times r}$ in \eqref{eq:noframe} does not satisfy the frame condition \eqref{eq:id} due to the insufficient number channels,
i.e. $r< d$. However, we still have the following matrix equality that is essential for deep convolutional framelet expansion \cite{ye2017deep}:
$$\hank_d(f) = \tilde \Phi \Phi^\top \hank_d(f)  \Psi \tilde \Psi^{\top}.$$
Accordingly, we can define a space  $\Hc_r$ by collecting signals 
that can be decomposed to the
single layer deep convolutional framelet expansion:
\begin{eqnarray*}
\Hc_r =\left\{ f\in \Rd^n~|~ f = \left(\tilde\Phi C\right) \circledast \nu(\tilde \Psi), 
C = \Phi^\top \left( f \circledast \overline \Psi  \right) \right\} 
\end{eqnarray*}
Then, the regression problem  in \eqref{eq:fcost}   can be equivalently represented by
\begin{eqnarray}\label{eq:newcost}
\min_{ f \in \Hc_r}   \|f^* - f\|^2   \ ,
\end{eqnarray}
which implies that the explicit rank condition is embedded as a single layer convolutional framelets.

However, \eqref{eq:newcost} holds for any signals that can be represented by arbitary $(\Phi,\tilde\Phi)$ and $(\Psi,\tilde\Psi)$ satisfying \eqref{eq:noframe},
and we should find ones that are optimized for given data.
In our deep convolutional framelets,  $\Phi$ and $\tilde \Phi$ correspond to the generalized pooling and unpooling which are chosen based on
the application-specific knowledges \cite{ye2017deep}, so we are interested in only estimating the filters $\Psi$, $\tilde \Psi$.
Then, the main goal of the neural network training  is to learn   ($\Psi$, $\tilde \Psi$) from training data  $\{(f_{(i)}, f_{(i)}^*)\}_{i=1}^N$ assuming that
$\{f_{(i)}^*\}$ are associated with rank-$r$ Hankel matrices.
Thus, \eqref{eq:newcost} can be modified for the training data as follows:
\begin{eqnarray}\label{eq:newcost2}
\min_{\{f_{(i)}\}\in \Hc_r}   \sum_{i=1}^N\|f_{(i)}^* - f_{(i)}\|^2  
\end{eqnarray}
which can be converted to the neural network training problem:
\begin{eqnarray}\label{eq:training}
 \min_{( \Psi, \tilde\Psi)}  \sum_{i=1}^N\left\|f_{(i)}^* - \Qc(f_{(i)};\Psi,\tilde\Psi)\right\|^2
 \end{eqnarray}
where
\begin{eqnarray}
\Qc(f_{(i)};\Psi,\tilde\Psi)&=& \left(\tilde\Phi C[f_{(i)}]\right) \circledast \nu(\tilde \Psi) \label{eq:Qcf} \\
\quad C[f_{(i)}] &=&
\Phi^\top \left( f_{(i)} \circledast \overline \Psi  \right) . \label{eq:Cf}
\end{eqnarray}
The idea can be further extended to the multi-layer deep convolutional framelet expansion with nonlinearity. Then,
\eqref{eq:training} is equivalently to the standard neural network training.
Once the network is fully trained, the inference for a given noisy input $f$ is simply done by
$\Qc(f;\Psi,\tilde\Psi)$, which is equivalent to find a denoised solution. 
Therefore, using deep convolutional framelets with insufficient channels, we do not need an explicit shrinkage operation and
the update equation \eqref{eq:update} can be replaced by
\begin{eqnarray}\label{eq:iter}
f_{n+1} = \mu  g + (1-\mu  )  \Qc(f_n;\Psi,\tilde\Psi) \ ,  
\end{eqnarray}
where 
$\Qc(f_n)$ is the deep convolutional framelet output.

However, one of the main differences of \eqref{eq:iter} from \eqref{eq:update} is that our deep convolutional framelet does not
satisfy the tight frame condition that is required to guarantees the convergence of \eqref{eq:update}.
Therefore,  to guarantee the convergence, we need to relax the iteration using Krasnoselskii-Mann (KM) method \cite{bauschke2011convex}
as described in Algorithm~\ref{alg:Pseudocode}.
Then, using the standard tools of proximal optimization \cite{bauschke2011convex}, we can show that
 the sequence generated by Algorithm~\ref{alg:Pseudocode} converges to a fixed point.
\begin{theorem}\label{thm:convergence}
There exists a parameter $\mu \in (0,1)$ such that the deep convolutional framelet denoising algorithm in Algorithm~\ref{alg:Pseudocode} converges to a fixed point.
\end{theorem}
\begin{proof}
See Appendix B.
\end{proof}

\begin{algorithm}
\caption{Pseudocode implementation.}
\label{alg:Pseudocode}
\begin{algorithmic}[1]
\State Train a deep network $\Qc$ using training data set.
\State  Set $0\leq \mu\leq 1$ and  $0<\lambda_n<1, \forall n$.
\State Set initial guess of $f_0$ and $f_1$.
\For{$n=1,2, \dots,$ until convergence}
   \State  $q_n:= \Qc(f_n) $
    \State $\bar f_{n+1} := \mu  g + (1-\mu) q_n$
    \State $f_{n+1} := f_n + \lambda_n (\bar f_{n+1}- f_n )$
    \EndFor
\end{algorithmic}
\end{algorithm}

Algorithm~\ref{alg:Pseudocode} corresponds to a recursive neural network (RNN) as shown in Fig.~\ref{fig:rnn}, which is related to an iterative
network in
\cite{wu2017iterative}.
If we use $\mu=0,\lambda_n=1$, the first iteration of Algorithm 1 corresponds to 
a feed-forward deep convolutional framelet denosing algorithm, which is also important by itself.
In Experimental Results, we show the improvement using  RNN.
However, our feed-forward network is much faster with compatible image quality. Thus, we believe that both algorithms are useful in practice.
Both algorithms have the same neural network backbone $\Qc(\cdot)$, which will be described in detail in the following section.

\begin{figure}[!htb]
\centering
\includegraphics[width=5.8cm]{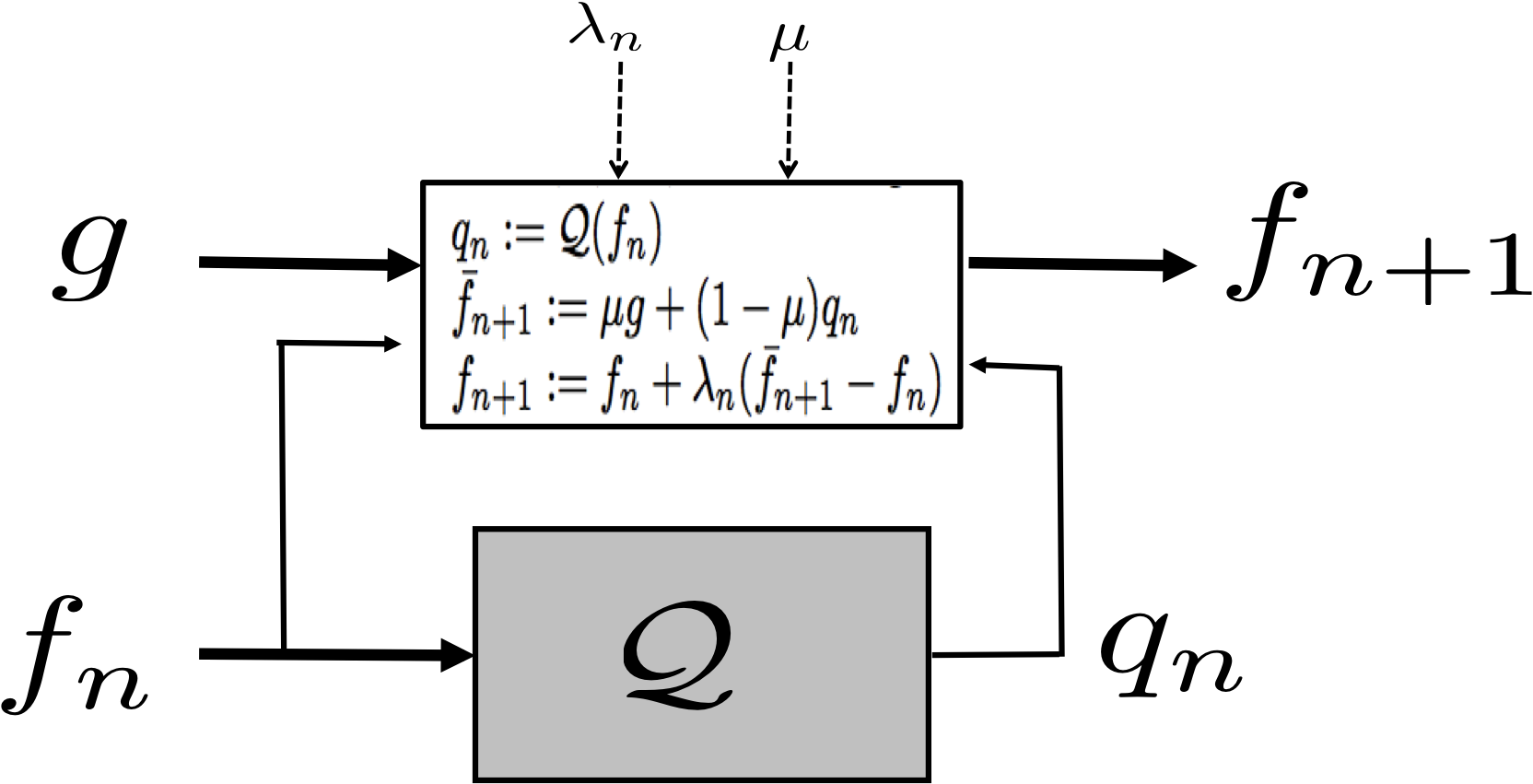}
\caption{Proposed RNN structure for deep convolutional framelet denoising.}
\label{fig:rnn}
\end{figure}

\subsection{Optimizing the Network Architecture}

In order to  have the best denoising performance in frame-based denosing,
the frame bases should have good energy compaction properties.  
For example, due to the vanishing  moments of wavelets, wavelet transforms can annihilate the smoothly varying signals while maintaining the image edges, thus resulting
in good energy compaction.
Thus, wavelet
frames such as contourlets \cite{zhou2005nonsubsampled}  are often used for denoising.
Furthermore, low-dose X-ray CT images exhibit streaking noise, so
the  contourlet transform \cite{zhou2005nonsubsampled}  is good for detecting the streaking noise patterns by
representing the directional edge information of X-ray CT images better.
Thus, we are  interested in using WavResNet \cite{kang2017wavelet} that employs the contourlet transform  \cite{zhou2005nonsubsampled}.
The proposed WavResNet architecture is illustrated in Fig.~\ref{fig:network_architecture}.
WavResNet has three unique
components: contourlet transform, concatenation, and skipped connection.
WavResNet is an extension of  our
prior work \cite{kang2017deep} that has similar network architecture except that
residuals at each subband are estimated by the neural network  \cite{kang2017wavelet}.
In this paper,   we provide a new interpretation of  WavResNet using the theory
of deep convolutional framelets \cite{ye2017deep}.

\begin{figure}[!hbt]
\centering
\includegraphics[width=9cm]{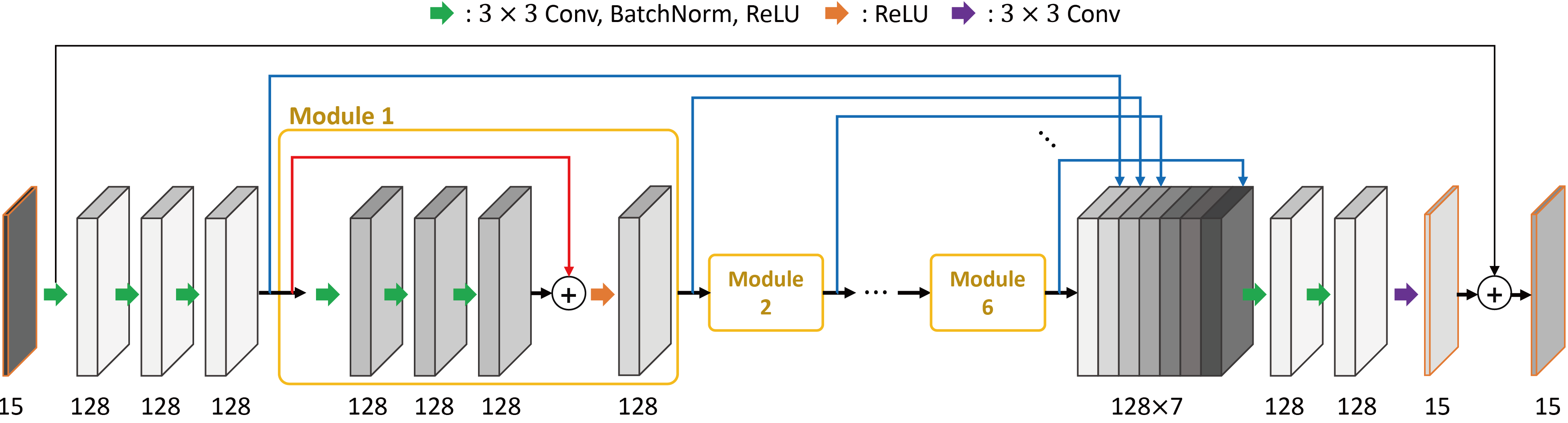}
\vspace*{-0.5cm}
\caption{The proposed WavResNet backbone (i.e. $\Qc(f)$ in Algorithm 1 and Fig.~\ref{fig:rnn}) for low-dose X-ray CT restoration.
 }
\label{fig:network_architecture}
\end{figure}


Specifically, for a given signal $f\in \Rd^n$,  the directional subband transform
  $\{T_k\}_{k=1}^p, T_k \in \Rd^{n\times n}$ in contourlet transform  satisfies the resolution of identity:
\begin{eqnarray}
\sum_{k=1}^p \tilde T^{(k)\top}T^{(k)} &=& I_{n\times n},~ \label{eq:T} 
\end{eqnarray}
which implies that there exists inverse transform $\{\tilde T_k\}_{k=1}^p$ to utilize all subband signals to recompose the original signal.
Thus,  the signal $f$ can be decomposed into directional components:
$$f = \sum_{k=1}^p \tilde T^{(k)\top}T^{(k)}f = \sum_{k=1}^p \tilde T^{(k)\top} f_k$$
where 
$f_k = T^{(k)}f \in \Rd^n$
corresponds to the $k$-th subband signals.
Then, our goal is to obtain the deep convolutional framelet representation of
the input matrix
$$Tf := \begin{bmatrix} T^{(1)}f & \cdots & T^{(p)} f\end{bmatrix}. $$
Because this multi-channel input size is big, we further decompose the signal using patch extraction operator
$$Z_l  = P_l T f = \begin{bmatrix} P_lT^{(1)}f & \cdots & P_lT^{(p)} f\end{bmatrix}. $$
where $\{P_l\}_l$ denotes the (overlapping) patch that has the same spatial location across all subbands.
The WavResNet does not use any pooling, i.e. $\Phi=\tilde\Phi = I_{n\times n}$, since the global correlations have been removed using the contourlet transform.
Thus,  by inserting each $Z_l$ in \eqref{eq:decZ} and \eqref{eq:encZ}, we have
\begin{eqnarray}\label{eq:Zdexp}
Z_l 
&=& C \circledast \nu(\tilde \Psi) \\
C &=&  \left( Z_l \right)\circledast  \overline\Psi  \  .   
 \end{eqnarray}
The successive layers are similarly implemented using the standard multi-channel convolution in CNN.
The resulting patch-by-patch CNN processing are performed on all parts of images, and the final results are obtained by  averaging.

Another important  component of WavResNet is the {\em boosting} using the concatenation layer.  This is closely related to the boosting scheme in classification that combines multiple weak classifiers to obtain 
a stronger classifier \cite{schapire1998boosting}.
Specifically, suppose that perfect recovery (PR) condition satisfies for all cascade of encoder-decoder network.
Then, the recovery condition 
for deep convolutional framelets up to $L$-layer can be written by
\begin{eqnarray*}
Z_l 
&=& {C}^{(1)}\circledast \nu\left(\tilde \Psi^{(1)} \right) \\
&\vdots & \\
Z_l 
&=& {C}^{(L)}\circledast \nu\left(\tilde \Psi^{(L)} \right)  \cdots\circledast \nu\left(\tilde \Psi^{(1)} \right) 
\end{eqnarray*}
where
\begin{eqnarray}\label{eq:Cenc}
C^{(i)}  &=&\begin{cases}  \left(C^{(i-1)} \circledast \overline\Psi^{(i)}\right),  & 1\leq i \leq L \\
Z_l, & i=0 \end{cases}
\end{eqnarray}
and the superscript $^{(i)}$ denotes the $i$-th layer.
Thus, for a given intermediate encoder output $\{C^{(l)}\}_{l=1}^L$, by denoting $h^{(l)} := \nu\left(\tilde \Psi^{(l)} \right)  \cdots\circledast \nu\left(\tilde \Psi^{(1)} \right)$, 
a {\em boosted} decoder can be constructed by  combining  multiple decoder representation:
\begin{eqnarray}
Z_l &=& \sum_{i=1}^L  w_i \left({C}^{(i)}\circledast  h^{(i)}\right), \quad
\end{eqnarray}
where $\sum_{i=1}^L w_i = 1$.
This procedure can 
be  performed using  a single multi-channel convolution after concatenating encoder outputs, as shown in Fig.~\ref{fig:newnet}(a).
In Experimental Results, we will show that this provides improved denoising performance thanks to the boosting effect.

Finally,  WavResNet has the skipped connection \cite{he2016deep} as shown in Fig.~\ref{fig:newnet}(b).
In order to understand the role of the skipped connection, 
 note that the low-dose input $f_{(i)}$ is contaminated with noise so that  it can be written by
$$f_{(i)} = f_{(i)}^*+h_{(i)}, $$
where $h_{(i)}$ denotes the noise components and $f_{(i)}^*$ refers to the noise-free ground-truth.
Then,  the network training \eqref{eq:training} using the skipped connection can be equivalently represented as the network training to estimate
the artifacts:
\begin{eqnarray}\label{eq:training2}
 \min_{( \Psi, \tilde\Psi)}  \sum_{i=1}^N\left\|h_{(i)}- \tilde\Qc(f_{(i)};\Psi,\tilde\Psi)\right\|^2
 \end{eqnarray}
where 
\begin{eqnarray}\label{eq:tildeQ}
\tilde\Qc(f_{(i)};\Psi,\tilde\Psi) =\left(\tilde\Phi C[f_{(i)}^*+h_{(i)}]\right) \circledast \nu(\tilde \Psi),
\end{eqnarray}
$$C[f_{(i)}^*+h_{(i)}]=
\Phi^\top \left( (f_{(i)}^*+h_{(i)})\circledast \overline \Psi  \right).$$
Therefore, if we can find a convolution filter $\overline \Psi $ such that
it approximately annihilates the true signal $f_{(i)}^*$ \cite{vetterli2002sampling}: 
\begin{eqnarray}\label{eq:fanal}
 f_{(i)}^*\circledast \overline \Psi \simeq 0  &\Longrightarrow C[f_{(i)}^*+h_{(i)}] \simeq C[h_{(i)}]
\end{eqnarray}
%
then  we  can find the decoder filter $\tilde\Psi$ such that
\begin{eqnarray*}
\left(\tilde\Phi C[h_{(i)}]\right) \circledast \nu(\tilde \Psi) 
&=&\left(\tilde\Phi   \Phi^\top \left( (h_{(i)})\circledast \overline \Psi  \right)\right) \circledast \nu(\tilde \Psi) \\
&=& h_{(i)}\circledast \overline \Psi  \circledast \nu(\tilde \Psi) \\
&\simeq &  h_{(i)} \ .
\end{eqnarray*}
Thus, our deep convolutional framelet with a skipped connection can estimate the artifact $h_{(i)}$  and remove it  from $f_{(i)}$.
On the other hand, using the similar argument, we can see that if
\begin{eqnarray}
 h_{(i)}\circledast \overline \Psi \simeq 0   &\Longrightarrow C[f_{(i)}^*+h_{(i)}] \simeq C[f_{(i)}^*]
\end{eqnarray}
then
a deep convolutional framelet {\em without} the skipped connection can directly recover the ground-truth signal $f_{(i)}^*$,
i.e.
$\Qc(f_{(i)};\Psi,\tilde\Psi) \simeq f_{(i)}^*.$
Then, which one is better ?
In our case, the true underlying signal has lower dimensional structure  compared to
  the random  CT noises,
so the annihlating filter relationship in \eqref{eq:fanal} is more easier to achieve \cite{vetterli2002sampling}.
Therefore,  we use the skipped connection as shown in   Fig.~\ref{fig:newnet}(b).

\begin{figure}[!bt] 
\center{\includegraphics[width=8cm]{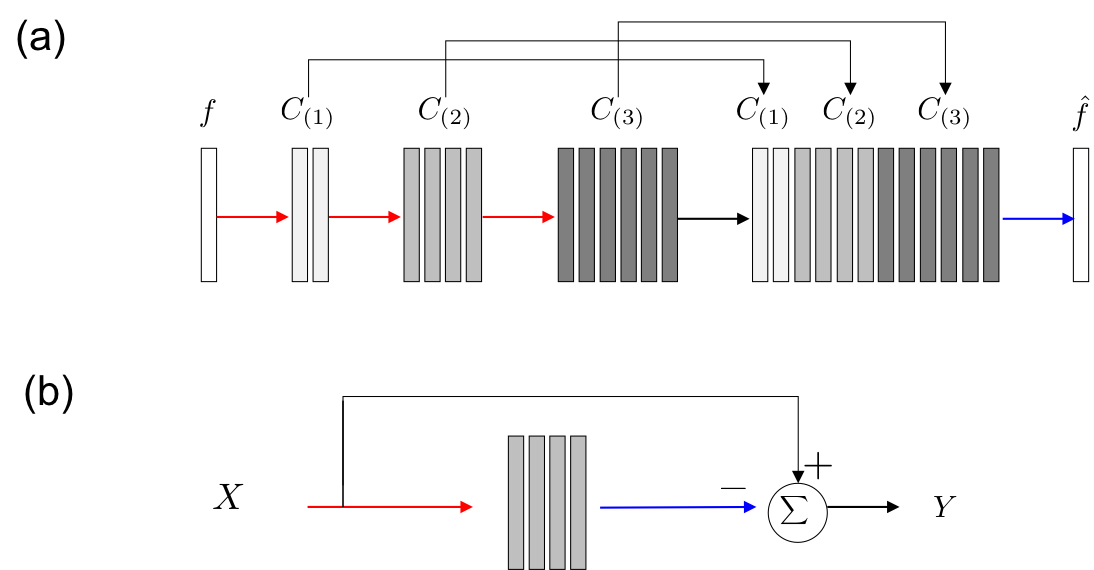}}
\caption{(a) Concatenation layer, and (b) skipped connection.}
\label{fig:newnet}
\end{figure}

By combining the contourlet transform, boosting and skipped connection,  we conjecture that  WavResNet can represent
the signal much more effectively which makes the deep convolutional framelet denoising effective.

\begin{figure}[!h]
\includegraphics[width=3.5in]{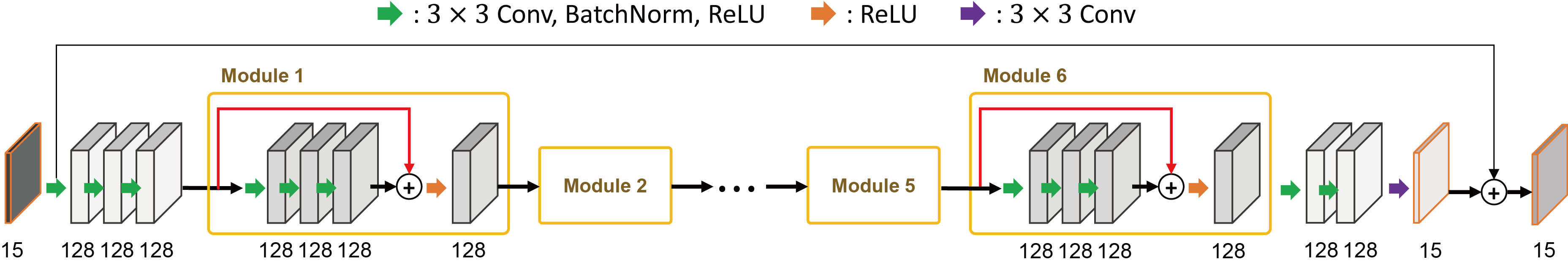}
\vspace*{-0.5cm}
\caption{A symmetric network architecture to investigate the importance of boosting layers in WavResNet.}
\label{fig:reference_network}
\end{figure}

\section{Method}
\label{sec:method}

\subsection{Proposed network architecture}

In  Fig. \ref{fig:network_architecture},
 we first apply non-subsampled contourlet transform to generate 15 channels inputs \cite{zhou2005nonsubsampled}.
There is no down-sampling or up-sampling in the filter banks; thus, it is a shift invariant.
We used 4 level decomposition and 8, 4, 2, 1 directional separations for each level, which produces the total 15 bands.
Thus, we have 15 subband channels. 

The first convolution layer uses 128 set of $3\times3\times15$ convolution kernels to produce 128 channel feature maps.
The shift invariant contourlet transform allows the patch processing and we used $55\times55\times15$ patches for the training and inference.
Then, the final contourlet coefficients are obtained by taking patch averaging.
Based on the calculation in \cite{ye2017deep},  a sufficient condition to meet the  PR is that
 the number of output channel should  be 270,  which is bigger than 128 channels. 
Thus, the first layer performs a low-rank approximation of the first layer Hankel matrix.
Then, the following convolution layers use two  $3\times3\times128$ convolution kernels, which is again believed to  perform low
rank approximation of the extended Hankel matrix approximation. Later, we will provide an empirical result showing that the singular
value spectrum of the extended Hankel matrix indeed becomes compressed as we go through the layers.

We have 6 set of main module composed of 
 3 sets of convolution, batch normalization, and ReLU layers, and 1 bypass connection with a convolution and ReLU layer.
Finally, as shown in Fig. \ref{fig:network_architecture}, 
our network has  the end-to-end bypass connection \cite{ye2017deep} so that  we can directly estimate
the noise-free contourlet coefficients while exploiting the advantages of skipped connection \cite{he2016deep}.
Another uniqueness of the proposed network is that it has the concatenation layer  as shown in Fig.~\ref{fig:newnet}(a).
Specifically, our network concatenates the outputs of the individual modules, which is followed by the convolution layer with 128 set of $3\times3\times896$ convolution kernels. 
As discussed before, this corresponds to the signal boosting using multiple signal representation.
In optimization aspect, this  also provides various paths for gradient back-propagation.
Finally, the last convolution layer uses 15 sets of $3\times3\times128$ convolution kernels. 
This may correspond to the pair-wise decoder layers with respect to the first two convolutional layers.

\subsection{Network training}

We trained two networks:  a feed-forward network and an RNN. 
We applied stochastic gradient descent (SGD) optimization method to train the proposed network.
The size of mini-batch was 10.
The convolution kernels were initialized by random Gaussian distribution.
The learning rate was initially set to 0.01 
and decreased continuously down to $10^{-7}$. 
The gradient clipping was employed in the range $[-10^{-3},10^{-3}]$ to use a high learning rate in the initial steps for  fast convergence.
For data augmentation, the training data were randomly flipped horizontally and vertically.
Our network was implemented using MatConvNet \cite{vedaldi2015matconvnet} in MATLAB 2015a environment (Mathworks, Natick).

%

{The training processes are composed by three stages.
In stage 1, we trained the network using original database $DB_0$ which consists of a pair of quarter-dose and routine-dose CT images.
After the network converged initially, stage 2 is proceeded sequentially.
In stage 2, we add databases $DB_i$ gradually which consists of quarter-dose input,  inference results from $\Qc_k(f_i)$, and routine-dose CT images.
Here, $\Qc_k$ dentoes the trained network until $k$-th epochs and $f_i$ is the $i$-th inference results.
Finally, in stage 3, we added a database whose both input and target images are routine-dose CT images. 
The theoretical background of such training is from the framelet nature of deep convolutional neural network \cite{ye2017deep}. Specifically, the neural network training is to learn the framelet bases from the training data that has the best representation of the signals. Thus, the learned bases should be robust enough to have near optimal representation for the given input data set. In our KM iteration, each iterative steps provided the improved images, which needs to be fed into the same neural network. Thus, the framelet bases should be trained to be optimal not only for the strongly aliased input but also for the near artifact-free images. 
The resulting network was used for both our RNN and feed-forward networks.
}

\begin{figure*}[!t]
\centering
\includegraphics[width=6.5in]{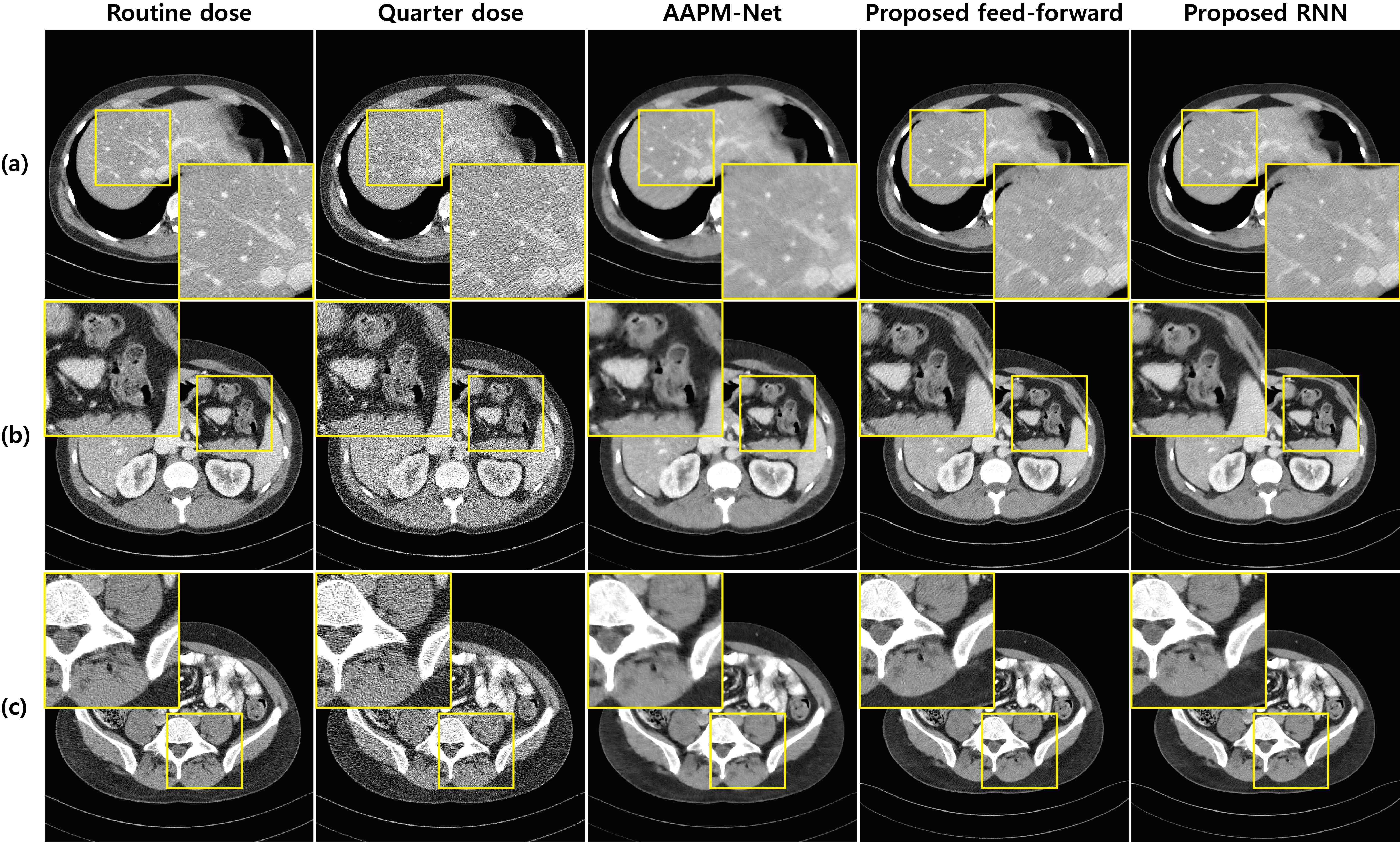}
\vspace*{-0.2cm}
\caption{Transverse view restoration results with routine-dose and quarter-dose images.
AAPM-Net is the algorithm which we applied to the ``2016 Low-Dose CT Grand Challenge".
Intensity range is (-160,240) [HU] (Hounsfield Unit).
(a) Example of liver, (b) example of intestine and (c) example of pelvic bone.}
\label{fig:result_cnn}
\end{figure*}

\subsection{Training dataset}

We used projection data obtained from ``2016 Low-Dose CT Grand Challenge".
The raw projection data were measured by a 2D cylindrical detector that moves along a helical trajectory using a z-flying focal spot \cite{flohr2005image}.
These projections were approximated into fanbeam projection data by a single slice rebinning technique \cite{noo1999single}.
We reconstructed X-ray CT images using conventional filtered backprojection algorithm.
The number of pixels in X-ray CT images is $512\times512$ and the slice thickness is 3mm.
We have 9 patient data sets of routine dose and quarter dose data for the training.
Eight patient data were used for the training and validation, and the remaining one patient data was used for testing.
Among the eight patient data, we used 3236 slices for the training and the remaining 350 slices for the validation.

For phantom studies, we used CT image data from Seoul National University Bundang Hospital, Korea.
 The number of pixels is $512\times 512$ and the slice thickness is 4mm. 
 The network was trained using 50 patient data sets of routine-dose,  13\% dose, 25\% dose, and 50\% dose images.
This study received technical support from Siemens Healthcare (Erlangen, Germany) to simulate CT images at various low dose levels.
Low dose images were simulated by inserting Poisson noise into the projection data of routine dose and reconstructing images from those projection data using the filtered back projection method.
 The number of individual dose images is 7617 slices and we arbitrarily selected 540 slices and 60 slices for training and validation, respectively, for every 50 epoch.
To quantify the resolution and contrast at test phase, 
we used the Catphan 500 (The Phantom Laboratory, Salem, NY, USA) composed of various modules. The spatial resolution was evaluated using a high-resolution module (CTP528), and contrast-to-noise ratio (CNR) were evaluated using a low contrast module (CTP515). The contrast was calculated as the difference of mean CT numbers between a supra-slice target (contrast of 1.0\% and diameter of 15mm) and the adjacent background area. The noise was defined as the standard deviation of the CT numbers of the adjacent background area. Mean CT numbers and standard deviations were calculated  from  circular region-of-interest (ROI) having a diameter of 1 cm in the target and the adjacent background area. ROIs were placed at the exactly same locations on the images produced by different algorithms.

\subsection{Baseline algorithms}


We compared the proposed method with the other denoising algorithms such as BM3D \cite{dabov2007image}, MBIR regularized by total variation (TV), ALOHA \cite{jin2016random}, the image domain deep learning approach (RED-CNN) \cite{chen2017low}, and CNN with GAN loss \cite{yang2017low}.
MBIR regularized by TV was solved using an alternating direction method of multiplier (ADMM) \cite{ramani2012splitting} and Chambolle's proximal TV \cite{chambolle2004algorithm}.
The details of RED-CNN and GAN  were obtained from the original paper  and we have implemented them
accordingly \cite{chen2017low, yang2017low}.

To verify the improvement of the new algorithm, we perform comparative study with our previous deep network in wavelet domain for ``2016 Low-Dose CT Grand Challenge" \cite{kang2017deep}.
We call this network as AAPM-Net.
The main difference between the proposed one and the AAPM-Net comes from the definition of the target images.
In AAPM-Net, the target images was the original wavelet coefficients except the lowest frequency band. 
More specifically, in AAPM-Net, the lowest frequency band target is the residual, whereas the higher frequency band signals are the wavelet coefficients themselves.
Therefore, this is not a ResNet from the perspective of deep convolutional framelets.
On the other hand, in WavResNet,  the residual wavelet coefficients between the routine-dose and low-dose inputs are estimated for every subband.
In the current implementation, the final network output is the artifact-corrected images by subtracting the estimated artifacts using
the end-to-end skipped connection as shown in  Fig.~\ref{fig:network_architecture}.
In order to demonstrate the importance of signal boosting, we also implemented a symmetric network as illustrated in
Fig. \ref{fig:reference_network}.
Except for the concatenation layers, the symmetric network also has identical 6 modules structures with symmetric encoder and encoder structures.

\section{Experimental Results}
\label{sec:result}
\subsection{Comparison with AAPM-net}

\begin{figure}[t]
\centering
\includegraphics[width=4cm]{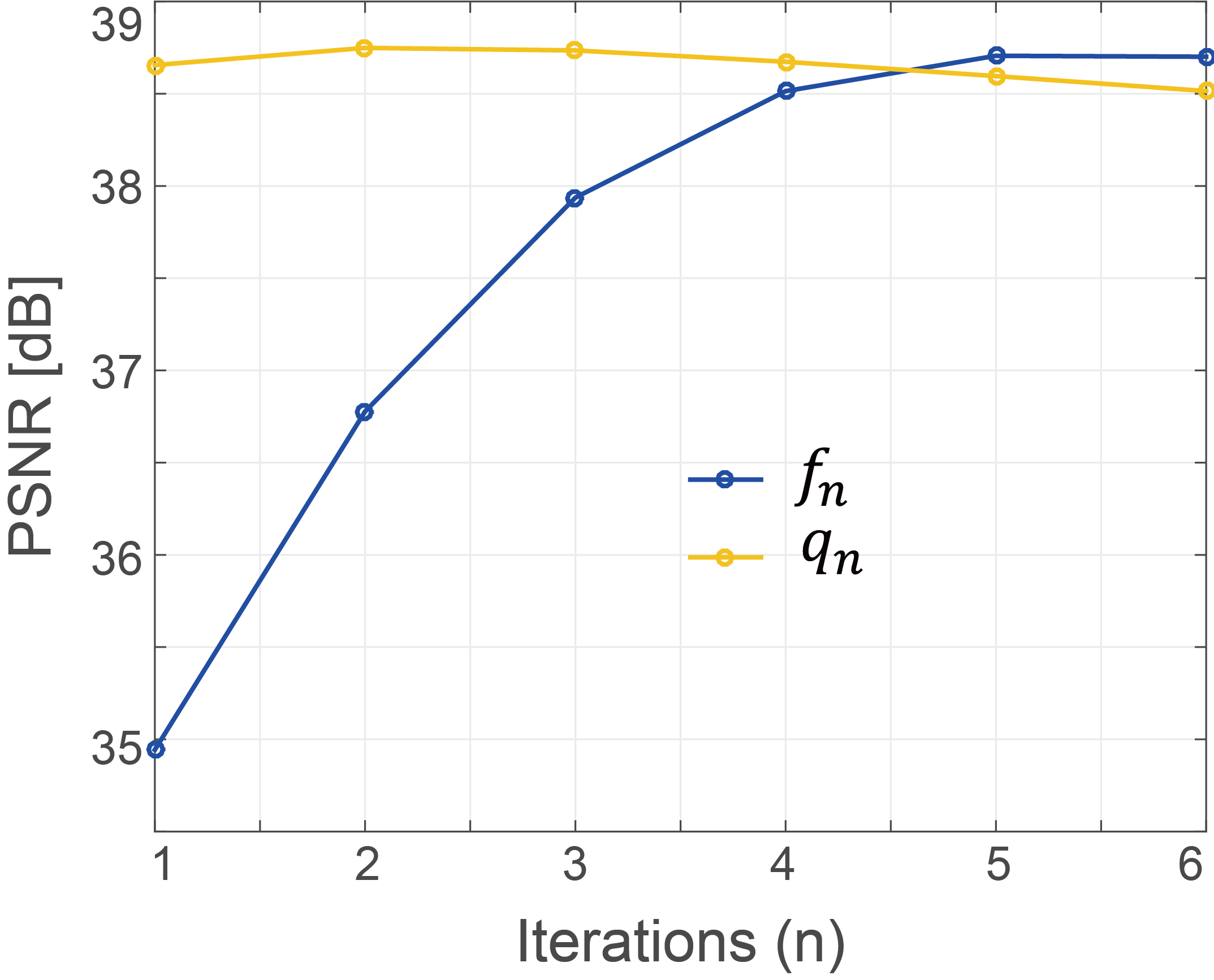}
\vspace*{-0.2cm}
\caption{PSNR values of restoration results according to iterations are plotted.}
\label{fig:psnr_iteration}
\end{figure}

\begin{figure}[!h]
\centering
\includegraphics[width=7cm]{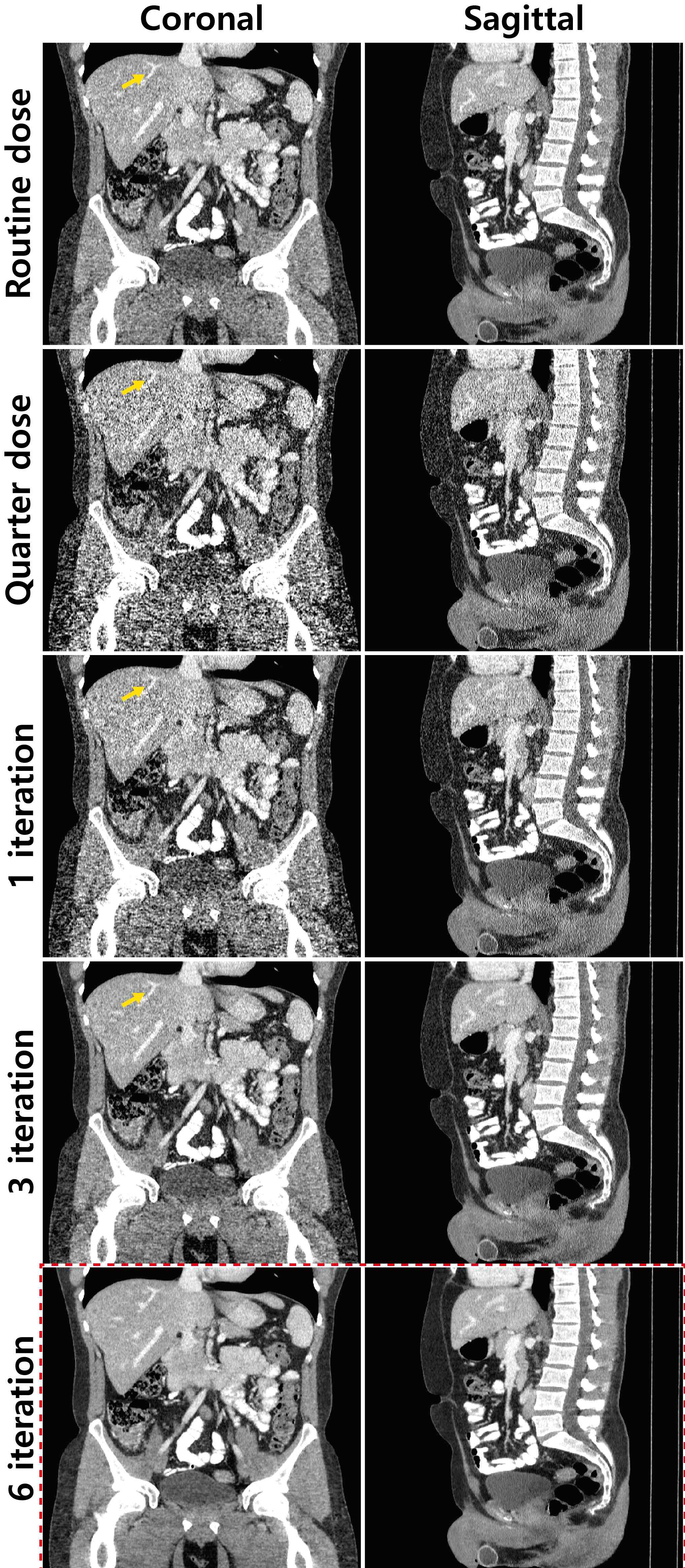}
\vspace*{-0.2cm}
\caption{Coronal and sagittal view restoration results along RNN iterations.
Intensity range is (-160,240) [HU]. The red dashline boxed images come from the last iteration.}
\label{fig:result_cnn_cor_sag}
\end{figure}

\begin{figure}[h!]
\centering
\includegraphics[width=6cm]{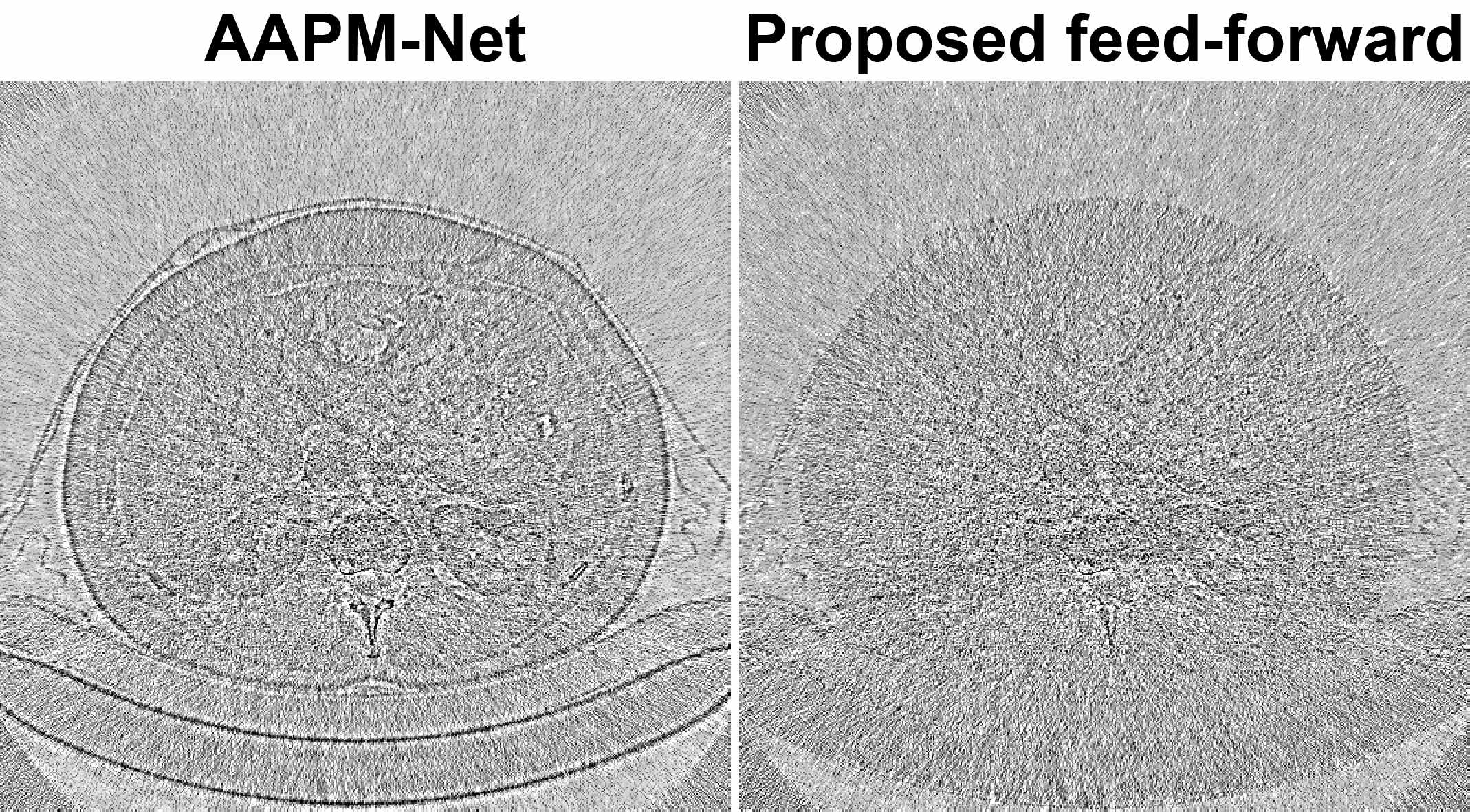}
\vspace*{-0.2cm}
\caption{Difference images between restoration images and routine dose image.
Images are same slice in the second row of Fig. \ref{fig:result_cnn}.
Intensity range is (-1100,-950) [HU].}
\label{fig:result_cnn_diff}
\end{figure}

To confirm the improvement over the AAPM-Net, we present the restoration images of one patient data in the test data set (see Fig. \ref{fig:result_cnn}).
This data have routine-dose images that can be used for subjective evaluation and objective evaluation using RMSE, PSNR, and SSIM.

In Fig. \ref{fig:result_cnn}, various kinds of slices such as liver and pelvic bones are described and the magnified images are expressed in the yellow boxes.
In AAPM-Net results, noise level was significantly reduced, but the results are blurry and loses some details.
On the other hand, the result of the proposed networks (feed-forward and RNN) results clearly shows that the improved noise reduction while maintaining the edge details and the textures which is helpful for diagnostic purpose.
More specifically, for the case of an liver image in Fig. \ref{fig:result_cnn} (a),
the proposed network results retain the fine details such as vessels in the liver and it has better sharpness than the AAPM-Net.
In  Fig. \ref{fig:result_cnn} (b), the detail of internal structure of intestine was not observed in quarter-dose images and AAPM-Net results, while they are well-recovered in proposed network results.
To examine the streaking noise reduction ability, we presented the slice which has the pelvic bone in Fig. \ref{fig:result_cnn}(c).
Proposed networks were again good at preserving the edge details such as inside region of the bones and the texture of the organ which located between the bones, while the streaking artifacts were completely removed. 
Among the feed-forward and RNN network structure,  the results by RNN suppresses more streaking artifacts compared to the feedforward neural network.
In Fig. \ref{fig:psnr_iteration}, the PSNR plots for $\Qc(f_n)$ and $f_n$  in Algorithm~\ref{alg:Pseudocode}
are illustrated.
The result shows that averaged PSNR values for RNN ($f_n$) are increased according to the iterations and it converged after 5 iterations.
On the other hand, PSNR values for the direct network output $\Qc(f_n)$ increase initially but started to decrease with iteration after the initial peak.
Eventually, the PSNR values for RNN surpass the feed-forward network.
This again confirms the convergence of the proposed RNN approach thanks to the KM iteration.  
However, our feed-forward network at the 1st iteration is also useful thanks to the computational advantage, so we provide the both results.
{In our KM iteration, 
$\mu$ is a parameter for incorporating the effect of the original low-dose image. This parameter also control the convergence behavior of KM iteration. The experimental results with various values of $\mu$ showed that the lower value tends to provide better results. However, if $\mu$ is less than 0.1, it did not converge and the image quality decreased in terms of iterations.  Thus, we set $\mu=0.1$ to get the best results while the algorithm retains the robustness.
}

The coronal and sagittal view of the restoration results by our RNN are described in Fig. \ref{fig:result_cnn_cor_sag}.
The quarter-dose images show that the noise levels are different depending on the slices.
The lower part of the images  exhibit a high noise level because the pelvic bones are included.
The results shows that the noise level was reduced gradually according to the iterations.
The last iteration reconstructed result maintains the edge details and textures which is helpful for diagnostic purpose.
The yellow arrow indicates the vessel in the liver and the last iteration result has better sharpness.
The proposed method can remove a wide range of noise levels and maintain the texture and edge information.
The difference images between the result images and routine-dose images in Fig. \ref{fig:result_cnn_diff} 
confirm the superiority of our method over AAPM-net.  
The difference images of proposed network only contains the noise of low-dose X-ray CT images, while the difference images of AAPM-Net also contains the edge information.


\begin{figure}[h!]
\centering
\includegraphics[width=8cm]{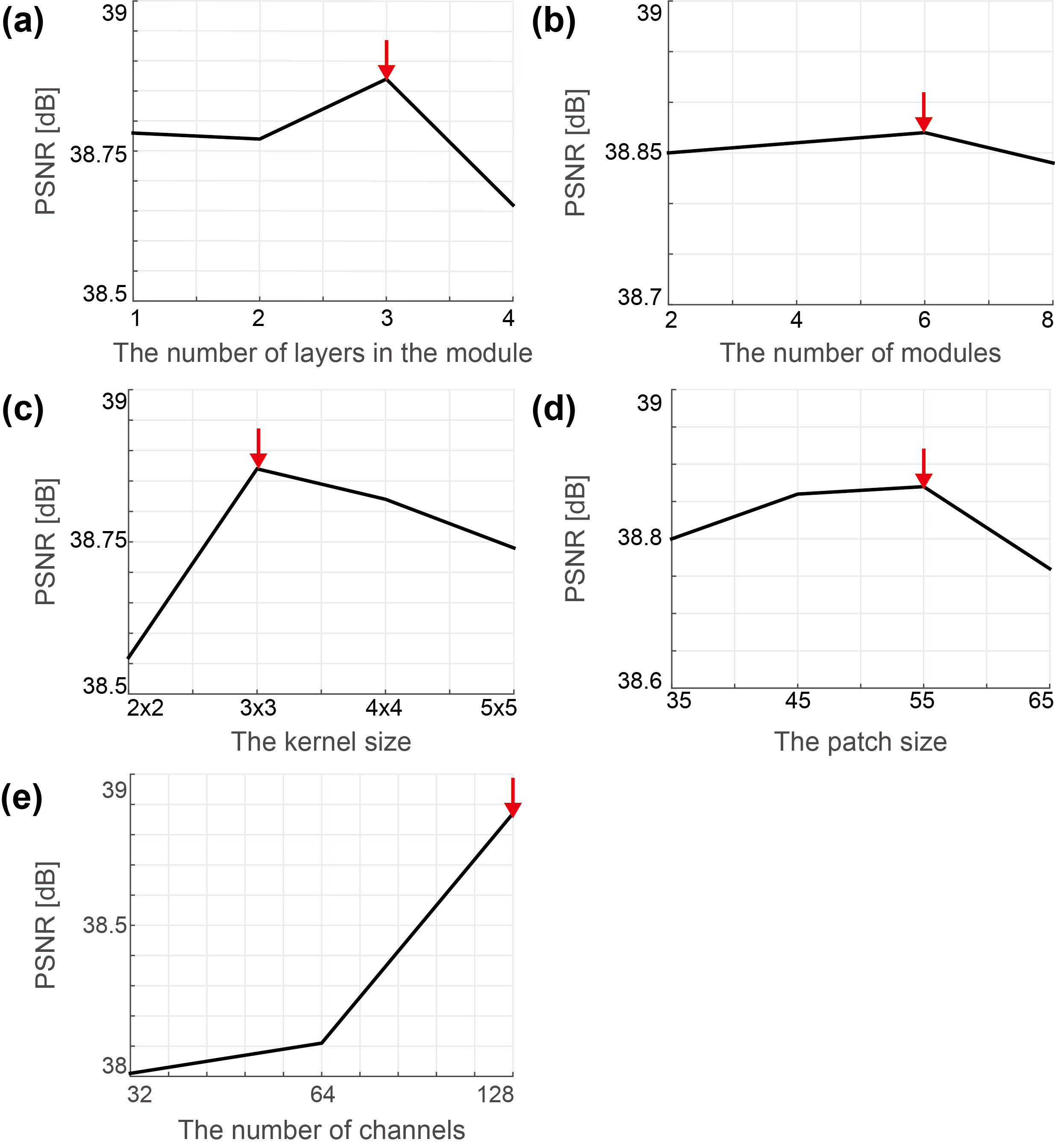}
\vspace*{-0.2cm}
\caption{{Performance dependency on various hyper-parameters of the network.}}
\label{fig:ablation_study}
\end{figure}

\subsection{Ablation study}

{
To demonstrate the advantages of the proposed method, we preformed the ablation study by
excluding some  structures from the network and applying the same training procedures. 
Table \ref{table_network_structure} presents the averaged RMSE, PSNR and SSIM index values of the results from 486 slices.
The qualitative results shows that proposed feed-forward and RNN network have the best results, and among them the RNN was better.
The PSNR and SSIM values of the symmetric network  in Fig.~\ref{fig:network_architecture} are  lower than those of the proposed methods, which confirms the signal boosting effect from the concatenation.
}

{
In addition, we have  investigated the effects of network hyper-parameters such as
the number of channels, the number of layers in module, the number of modules, the kernel size, and the patch size as
shown in
Fig. \ref{fig:ablation_study}.
Here,
network performance improves with more layers in each module until it reaches 3. With more than three layers we have found that the network is difficult to train due to many parameters to be optimized.
As the number of modules increases, network performance improves slightly at the expense of increased reconstruction time. Given the compromise between performance and reconstruction time, we used six modules for our network.
We have observed that the filter size $3\times 3$ gave the best result with reasonable processing time for real applications.
In addition, we  found that the patch size is not critical. However, the reconstruction time and its receptive field lead us to choose the patch size of $55\times 55$ for our network.
Finally, with more channels, the performance improved. But due to the memory requirement as well as to prevent overfitting, we chose 128 channels.
}

\begin{table}[!h]
\caption{Analysis of network structure}
\label{table_network_structure}
\centering
\resizebox{0.38\paperwidth}{!}{
\begin{tabular}{|c||c|c|c|}
\hline
 									& \textbf{RMSE}	& \textbf{PSNR} [dB] 	& \textbf{SSIM index} \\ \hline\hline
Exclude external bypass connection 	& 48.09			& 33.63					& 0.828 \\ \hline
Exclude concatenation layer (symmetric) & 28.43		& 38.20					& 0.893 \\ \hline \hline
Proposed feed-forward (128 channels)& 27.32			& 38.54					& 0.899 \\ \hline
Proposed RNN (128 channels)			& 26.90			& 38.70					& 0.893 \\ \hline
\end{tabular}}
\vspace*{-0.5cm}
\end{table}


\subsection{Low-rank approximation property}

To verify our theory that CNN is closely related to the Hankel matrix decomposition \cite{ye2017deep}, we performed
experiments to verify whether  the trained network imposes low-rank approximation of the Hankel matrix.
Specifically,  we constructed extended Hankel matrices using the output channel images from each module in Fig. \ref{fig:network_architecture}.
Then, we plotted the singular value spectrum  in Fig. \ref{fig:rank_exp}.
Blue dashed line is the result of the first module's extended Hankel matrix constructed from output feasture maps and the red solid line is the sixth module's extended Hankel matrix constructed from output features maps.
We observed that the singular value spectrum becomes compressed as we go through layers, which indicates
each layer of CNN performs the low-rank approximation of the Hankel matrix.
These results clearly suggests the link between trained CNN and the low rank Hankel matrix decomposition.

\begin{figure}[h!]
\centering
\includegraphics[width=5cm]{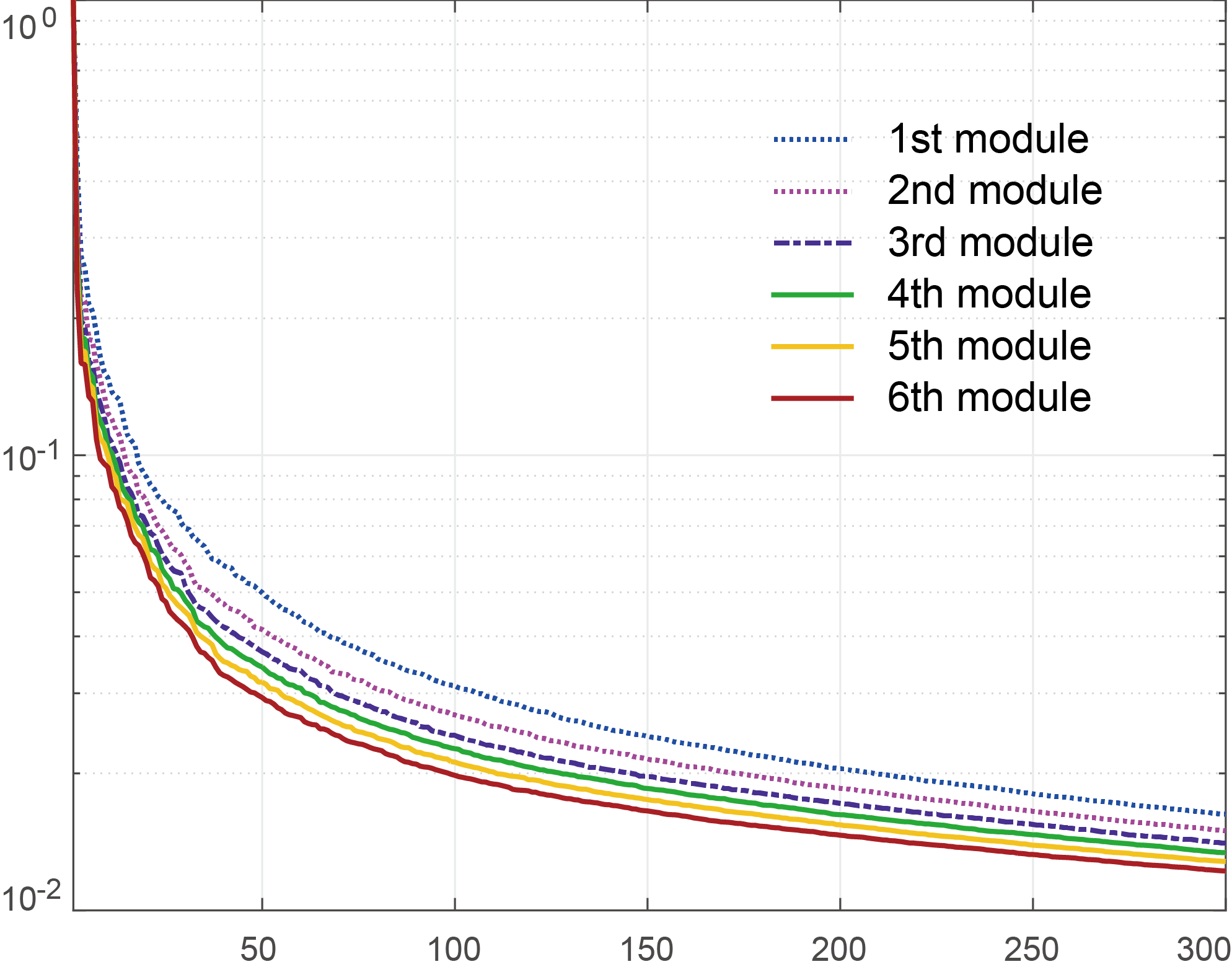}
\vspace*{-0.2cm}
\caption{The singular value spectrum of the extended Hankel matrix along layers.}
\label{fig:rank_exp}
\end{figure}

\subsection{Comparison with existing algorithms}

\begin{figure}[!hbt]
\centering
\includegraphics[width=9cm]{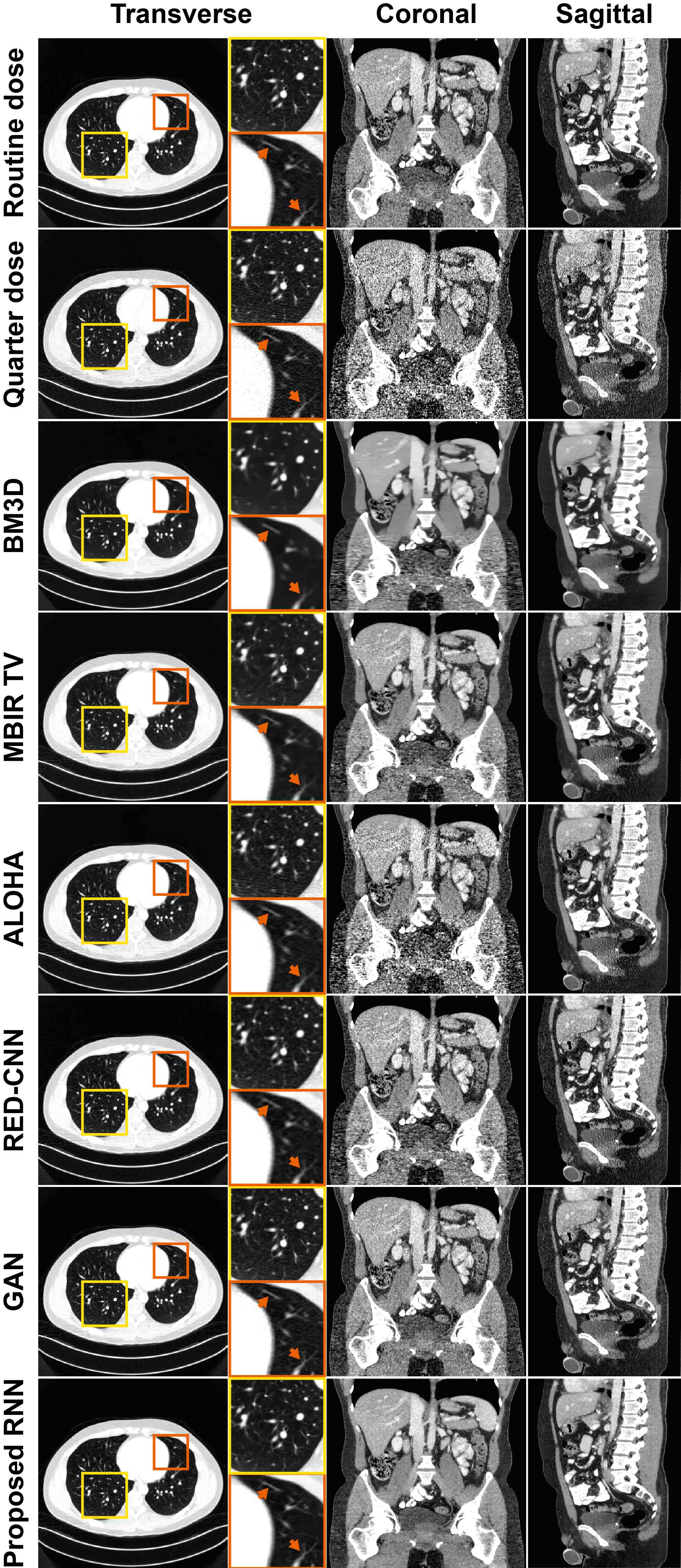}
\vspace*{-0.5cm}
\caption{Restoration results with routine-dose and quarter-dose image.
Transverse view restoration images' intensity range is adjusted to see the details in the lung.
Intensity range is (-1000,100) [HU].}
\label{fig:result_comparative}
\end{figure}

Fig. \ref{fig:result_comparative} shows the results by the comparative algorithms such as BM3D \cite{dabov2007image}, MBIR regularized by TV, ALOHA \cite{jin2016random}, RED-CNN \cite{chen2017low}, and GAN loss \cite{yang2017low}.
BM3D is a state-of-art of image denoising algorithm using nonlocal patch processing, MBIR is currently a standard algorithm of low-dose X-ray CT images and the 
RED-CNN is recently proposed deep network for low-dose X-ray CT and ALOHA is the latest low rank Hankel matrix method.

The intensity of the transverse view in Fig. \ref{fig:result_comparative} is adjusted to see inside structures of the lung.
The result of BM3D loses the details in the lung such as vessels and exhibited some cartoon artifact.
The result of MBIR appears a little blurred and textures are reconstructed incorrectly.
On the other hand, deep learning based denoising algorithms have better performances than the other algorithms.
However, RED-CNN results are somewhat blurry and exhibits remaining noises in the coronal view, 
while the proposed method provides clear restoration results.



\begin{table}[!h]
\renewcommand{\arraystretch}{1.7}
\caption{Execution time (mini-batch: $55\times55\times10$,slice: $512\times512$)}
\centering
\resizebox{0.35\paperwidth}{!}{
\begin{tabular}{|c||c|c|}
\hline
\textbf{Training} 	& \textbf{Time} [mini-batch/sec] 	& \textbf{Implementation environment} \\ \hline
RED-CNN 	& 0.19 	& MatConvNet, GTX 1080 Ti \\ \hline
Proposed 	& 0.44 	& MatConvNet, GTX 1080 Ti \\ \hline \hline

\textbf{Restoration} & \textbf{Time} [slice/sec] 	& \textbf{Implementation environment} \\ \hline
BM3D 		& 2.73 	& MATLAB, i7-4770 \\ \hline
MBIR TV 	& 9.45	& MATLAB, GTX 1080 \\ \hline
ALOHA 		& 1405 	& MATLAB, GTX 1080 \\ \hline
RED-CNN 	& 0.38 	& MatConvNet, GTX 1080 Ti \\ \hline
Proposed feed-forward 	& 2.05 	& MatConvNet, GTX 1080 Ti \\ \hline
\end{tabular}
}
\end{table}

With regard to the computation time, 
the CNN frameworks need learning  to train the networks. 
Our method took 20 hours to train the network through 3 stages as described in Section  \ref{sec:method}, and RED-CNN took 6 hours to train.
For the restoration step, the CNN framework is advantageous compared to the other classical algorithms such as BM3D or MBIR TV.
Our method takes approximately 2.05 seconds per slice for restoration which have $512\times512$ pixels with MATLAB implementation using a graphical processing unit (NVidia GeForce GTX 1080 Ti).

\subsection{Contrast and resolution loss study}

To evaluate the contrast and spatial resolution loss, we compared the various algorithms using Catphan phantom at various dose levels such as 13\%, 25\%, 50\% of the original
dose.
In Fig. \ref{fig:result_acr}, the representative restoration results from 25\% dose are illustrated, where
 the magnified areas are indicated by  yellow boxes. By visual inspection, we can see that  RED-CNN and the proposed method 
 preserve the resolution lines better than other methods, and among them our method was better at all resolution grids.

\begin{figure}[h!]
\centering
\includegraphics[width=8cm]{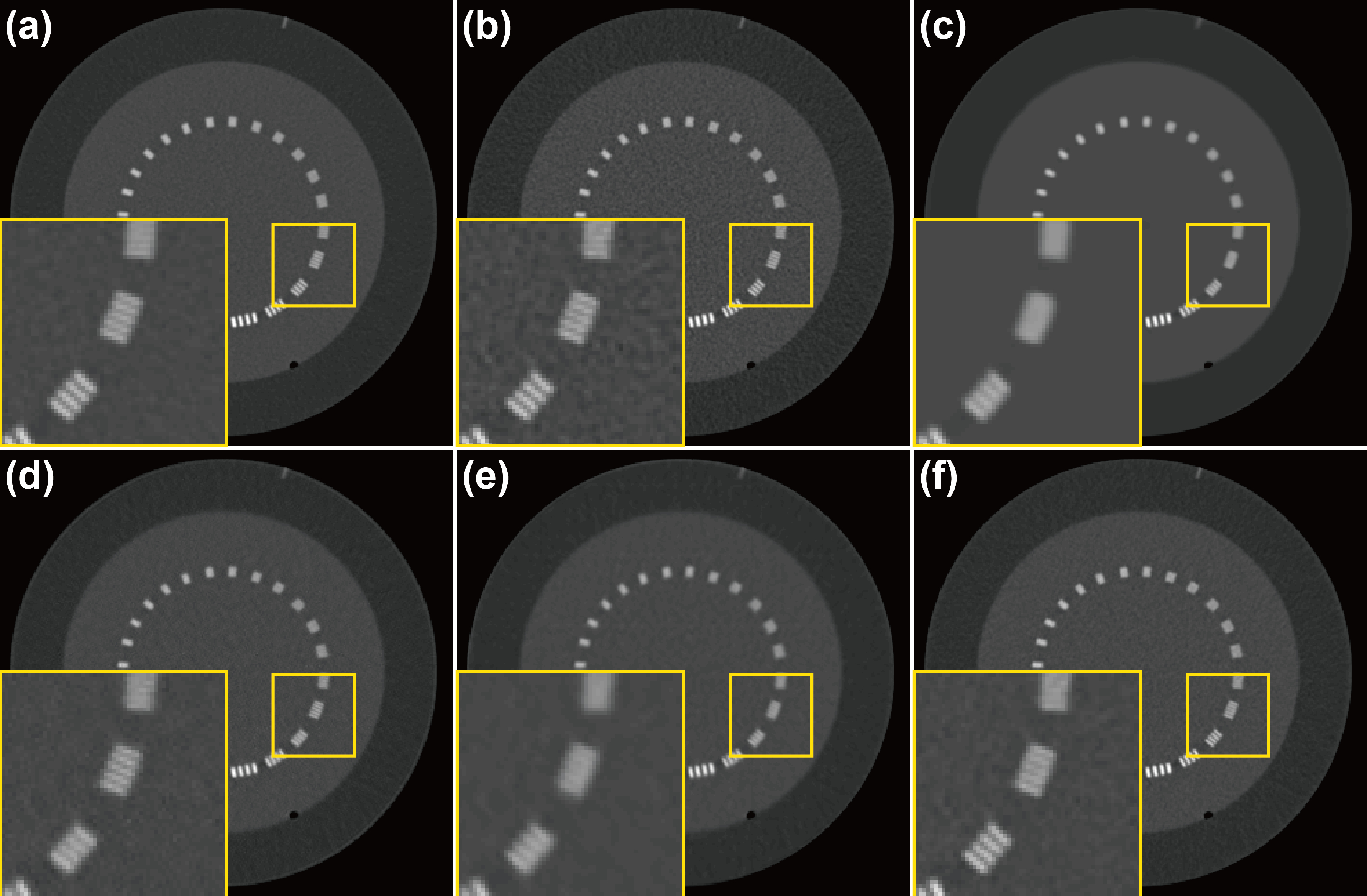}
\vspace*{-0.2cm}
\caption{
Reconstruction results of Catphan data from $\%25$ dose. Images by
 (a) routine dose, (b) quarter dose input, (c) BM3D, (d) RED-CNN, (e) GAN, and (f) the proposed feed-forward network.}
\label{fig:result_acr}
\end{figure}

To  investigate the spatial resolution loss at lose dose level, 
Fig.~\ref{fig:result_profile} illustrates the intensity profile along the two resolution grids in Catphan phantom at various dose level.
Down to quarter dose level,  the proposed feed-forward network does not exhibit significant resolution loss.
At 13\% dose,  we started to observe resolution loss especially at area (a). 
In addition, Table~\ref{tbl:contrast} shows the contrast  to noise (CNR) variations at various radiation dose levels.
%
The CNR values  of our method gradually decreases from 50\% to 13\%, but
they outperformed the CNR values of FBP results at the same dose levels. These results clearly confirm
the robustness of the proposed method at various dose levels.

\begin{figure}[h!]
\centering
\includegraphics[width=8.cm]{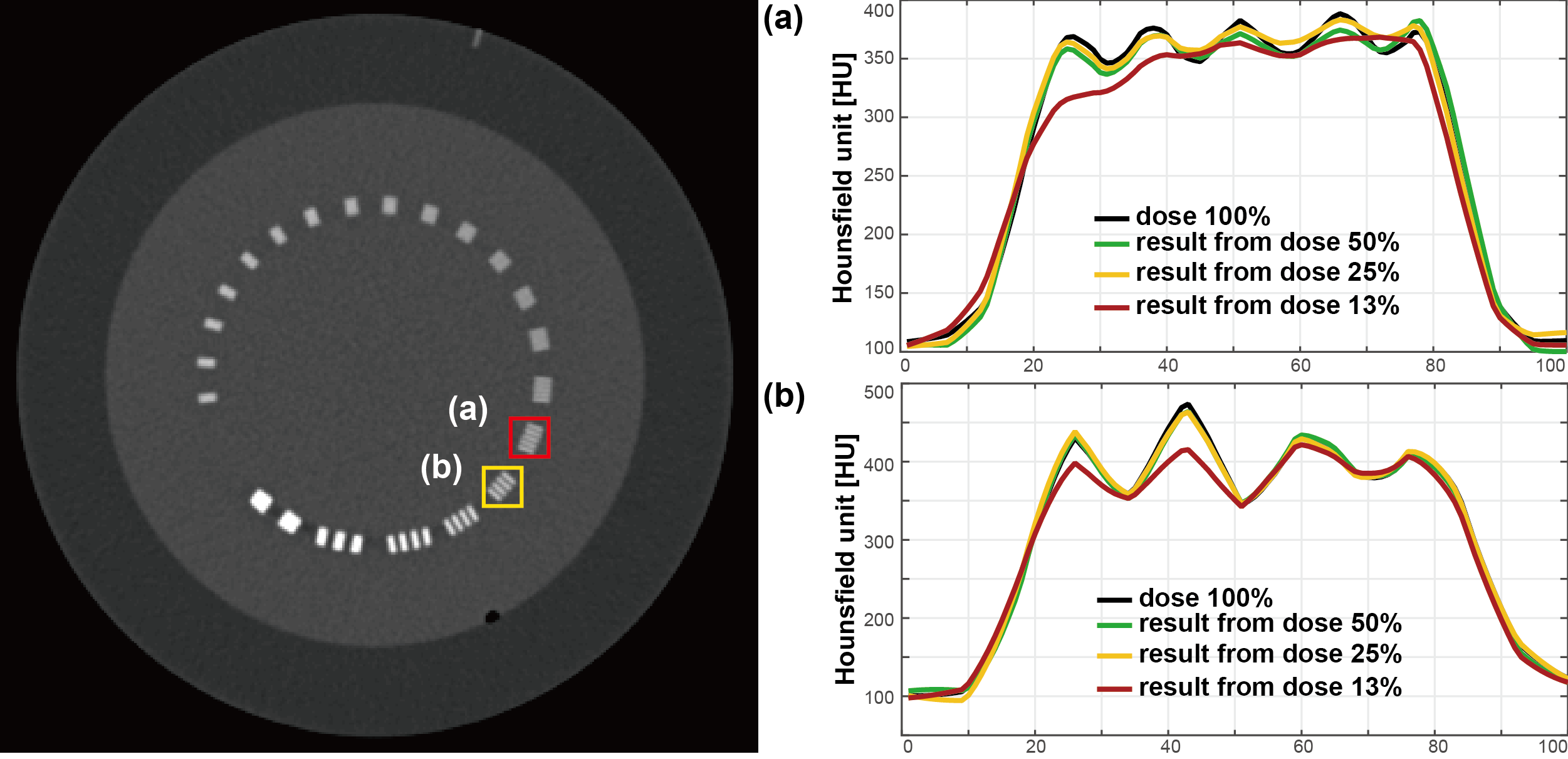}
\vspace*{-0.1cm}
\caption{
Resolution profile at various level by the proposed feed-forward network. 
}
\label{fig:result_profile}
\end{figure}


\begin{table}[!h]
\renewcommand{\arraystretch}{1.3}
\caption{CNR evaluation at various dose levels}
\label{tbl:contrast}
\centering
\resizebox{0.4\paperwidth}{!}{
\begin{tabular}{|c|c|c|c|c|c|c|c|}
\hline
Radiation dose	level 		& 100\%		&    \multicolumn{2}{c|}{50\%}	&   \multicolumn{2}{c|}{25\%} &   \multicolumn{2}{c|}{13\%} \\ \hline\hline
Alogrithm  & FBP & FBP & Proposed  & FBP & Proposed & FBP & Proposed \\
\hline
 Contrast & 8.49 &	8.44 	& 	8.44 & 	8.96 & 	 	9.46  &	10.85 	&	11.08 \\
Noise	& 6.38 & 8.94 	& 	6.24  &	13.69 	& 	9.46 	 & 17.26 	&	11.45 \\ \hline
CNR	       & 1.33 &  0.94 	  &	1.35  &	0.65 	&1.00 	& 0.63 	&	0.97 	\\ 
 \hline 
\end{tabular}
\vspace*{-0.5cm}
}
\end{table}

\subsection{Evaluation of lesion detection}
{
To verify the proposed method, 
we have performed task-driven experiment to evaluate the lesion detection performance by quarter-dose FBP, MBIR TV and a proposed method. We used 20 test data sets from the 2016 Low dose CT grand challenge. A board-certified radiologist (Won Chang) with seven years of experience in liver CT interpretation assessed the data set and recorded the exact locations of the lesions with blind to the reconstruction methods. The detection rates for solid focal hepatic lesions were compared using the McNemar test with Bonferroni correction and a difference with a $p>0.017$ was considered significant.
Since small number of radiologist involved in this study, it was not statistically significant, but the proposed method showed a significant higher lesion detection rate than FBP (73\% vs. 57\%, p-value=0.0412) and MBIR(73\% vs. 62\%, p-value=0.1336).
 With more radiologists involved, we are currently studying large scale statistical evaluation of the method, which will be reported later in a clinical journal.
%
%
}

\begin{table}[!bt]
\renewcommand{\arraystretch}{1.3}
\caption{{Evaluation of lesion detection}}
\label{table_lesion_detection}
\centering
\resizebox{.35\paperwidth}{!}{
\begin{tabular}{|c||c|c|c|c|}
\hline
 						& Ground-truth 	& Quarter-dose	& MBIR TV 	& Proposed \\ \hline\hline
The number of lesions 	& 37			& 21	 		& 23 		& 27 \\ \hline
Lesion detection rate 	& - 			& 57\% 			& 62\% 		& 73\% \\ \hline
\end{tabular}}
\end{table}

\section{Conclusion}
\label{sec:conclusion}

In this paper, we proposed a deep convolutional framelet-based denoising algorithm for low-dose X-ray CT restoration by synergistically combining the proven convergence
of the classical framelet-based algorithm and the expressive power of deep learning.
To provide the theoretical background for performance improvement, we employed the recent proposal of deep convolutional framelets that interprets a deep learning as a multilayer implementation of convolutional framelets with ReLU nonlinearity. 
Our theory resulted in two  network structures: a feed-forward and RNN architectures.
Moreover, by combining the redundant global transform, residual network  (ResNet) and signal boosting from concatenation layers,
the proposed feed-forward and RNN network provided significant improvement compared to the prior work by retaining the detailed texture.
Using extensive experimental results,
we showed that the proposed network is good at streaking noise reduction and preserving the texture details of the organs while the lesion information is maintained.

\section*{Acknowledgement}

The authors would like to thanks Dr. Cynthia MaCollough,  the Mayo Clinic, the American Association of Physicists in Medicine (AAPM), and grant EB01705 and EB01785 from the National
Institute of Biomedical Imaging and Bioengineering for providing the Low-Dose CT Grand Challenge data set.
This work is supported by Korea Science and Engineering Foundation, Grant number NRF-2016R1A2B3008104.
This study received technical support from Siemens Healthcare (Erlangen, Germany) to simulate CT images of various low dose levels.
We thank Seongyong Pak (Siemens Healthcare Ltd, Korea) for technical support on simulating CT images of low dose levels.
The funders had no role in study design, data collection and analysis, decision to publish, or preparation of the manuscript.

\appendices

\section{Mathematical Preliminaries}

 For a given mapping $T: D \rightarrow \Hc$,
the set of the {\em fixed points} of an operator $T:D \rightarrow D$ is denoted by
$\fix T = \{ x\in D~|~ Tx=x \}.$
 Then, $T$ is called 
{\em non-expansive}  if 
\begin{equation}\label{eq:nonexp}
\|Tx - Ty \| \leq  \| x-y \| ,\quad   \forall x,y \in D ,
\end{equation}
%
%
Then, we have the following convergence theorem for the non-expansive operator:
 \begin{theorem}[Krasnoselski-Mann algorithm]\cite{bauschke2011convex}  \label{thm:km}
 Let $D$ be a nonempty
closed convex subset of $H$, let $T : D \mapsto D$ be a nonexpansive operator such that
$\fix T\neq \emptyset$, let $(\lambda_n)$ be a sequence in $[0, 1]$ such that
$\sum_{n=1}^\infty \lambda_n(1-\lambda_n) = +\infty$ and let $f_0 \in D$.
Consider the following seqeunce:
\begin{eqnarray}\label{eq:myiter}
 f_{n+1} = f_n + \lambda_n  (T f_n - f_n).
 \end{eqnarray}
Then,  the sequence $f_n$ converges to a point in $\fix T$.
\end{theorem}

\section{Proof of Theorem~\ref{thm:convergence}}

Let the mapping $T$ be defined by
\begin{eqnarray*}
T(f) 
&:=& \mu  g + (1-\mu )  \Qc(f)  
\end{eqnarray*}
where  $\Qc$  is the deep convolutional framelet  network output.
Our goal is to show that the operator $T$ is non-expansive.
Note that
\begin{eqnarray*}
\|Tx - Ty \|  &=& \|(1-\mu )  \Qc (x) - (1-\mu ) \Qc(y)\| \\
&= &(1-\mu)\|\Qc(x) - \Qc(y) \| \\
&\leq & (1-\mu) \|\Qc'(z)\|\|x-y\|
\end{eqnarray*}
where  $\Qc'(z)$ denotes the Jacobian of the network at $z$, and
we use  the mean value theorem for the last inequality.
Therefore, if the Jacobian of the deep convolutional framelet is finite, we can choose $\mu=1-1/\max_{z\in D}\|\Qc'(z)\|$ such that
$$\|Tx - Ty \|  \leq \|x-y\|,$$
i.e. $T$ is non-expansive.
Thus, the remaining step is to show that $\max_{z\in D}\|\Qc'(z)\| < \infty$.
This is true because the convolutional framelets consists of convolutional filters with finite coefficients
and ReLU, so the Jacobian can be represented as the  product of the filter norms \cite{sokolic2017robust}.
 Thus,  by denoting $\bar f_{n+1} = T f_n$, Theorem~\ref{thm:km} informs that
 our KM iteration for deep convolutional framelet inpainting
converges. This concludes the proof.




\end{document}